%% file: main.tex
\documentclass[a4paper,fleqn]{cas-dc}

\usepackage[numbers]{natbib}

\usepackage{amsmath,amssymb,amsfonts}
\usepackage{amsthm}
\newtheorem{theorem}{Theorem}
\newtheorem{proposition}{Proposition}
\newtheorem{corollary}{Corollary}
\usepackage{algorithmic}
\usepackage{algorithm}
\usepackage{graphicx}
\usepackage{textcomp}
\usepackage{xcolor}
\usepackage{tikz}
\usetikzlibrary{shapes.geometric, arrows.meta, positioning, fit, backgrounds, calc}
\definecolor{clientblue}{HTML}{2166AC}
\definecolor{clientfill}{HTML}{EFF3FF}
\definecolor{servergray}{HTML}{636363}
\definecolor{serverfill}{HTML}{F7F7F7}
\definecolor{haogold}{HTML}{E6850E}
\definecolor{haofill}{HTML}{FFF3E0}
\definecolor{encgreen}{HTML}{1B7837}
\definecolor{dpviolet}{HTML}{7B3294}
\usepackage{booktabs}
\usepackage{multirow}
\usepackage{hyperref}
\usepackage{cleveref}

\begin{document}

\let\WriteBookmarks\relax
\def\floatpagepagefraction{1}
\def\textpagefraction{.001}

\shorttitle{Homomorphic Advantage Operator}
\shortauthors{A. M. Nadhir et al.}

\title[mode = title]{Homomorphic Advantage Operator: Stabilizing Reinforcement Learning Under Fully Homomorphic Encryption Constraints}

\author[1]{Abid Mohamed Nadhir}
\cormark[1]
\ead{abid-mohamednadhir@univ-eloued.dz}

\author[2,3]{Ahmad Al Hanbali}
\cormark[2]
\ead{ahmad.alhanbali@kfupm.edu.sa}

\author[1]{Beggas Mounir}
\cormark[1]
\ead{beggas-mounir@univ-eloued.dz}

\affiliation[1]{organization={El Oued University},
            city={El Oued},
            country={Algeria}}

\affiliation[2]{organization={King Fahd University of Petroleum and Minerals (KFUPM)},
            city={Dhahran},
            country={Saudi Arabia}}
\affiliation[3]{organization={Interdisciplinary Research Center For Smart Mobility and Logistics, KFUPM},
            city={Dhahran},
            country={Saudi Arabia}}

\cortext[1]{Corresponding author}

\begin{abstract}
Privacy-preserving machine learning presents significant deployment challenges on the cloud for intelligent systems with confidential data. Fully Homomorphic Encryption (FHE) offers a compelling solution for secure computation, preserving data confidentiality of cloud computations. However, applying FHE to reinforcement learning (RL) requires replacing non-linear operations with polynomial approximations, which diverge catastrophically due to a unique recursive error phenomenon known as the Bellman drift. This article introduces the Homomorphic Advantage Operator (HAO), a stabilization framework designed to prevent polynomial approximation divergence in FHE-based deep RL. HAO adapts the zero-mean centering projection from advantage-based value estimation directly to temporal-difference (TD) targets. This linear projection annihilates the uniform state-value baseline that drives the Bellman drift, maintaining per-state action rankings while requiring zero additional non-linear multiplicative depth and avoiding expensive ciphertext bootstrapping. The proposed HAO framework was evaluated using a three-tier experimental methodology, including a tabular Markov Decision Process (MDP), an encrypted CartPole environment using real CKKS cryptographic operations, and a 20-node logistics routing benchmark with dense continuous features. The results demonstrate that the proposed HAO strictly bounds network pre-activations within the safe polynomial approximation domain. The proposed HAO RL agents achieved 0\% boundary breaches across all random seeds used, whereas regularization alone (L2 weight decay and gradient clipping) breached the bound on 3 of 5 seeds and the unstabilized baseline did so in 83.8\% of episodes. Finally, HAO agents improve optimal policy accuracy by 18.0 percentage points in tabular domains and remain stable when DP-SGD-style Gaussian noise is added to the clipped gradients.
\end{abstract}

\maketitle

\section{Introduction}
\label{sec:intro}

Reinforcement learning (RL) has emerged as a powerful tool for optimizing complex, dynamic systems, from autonomous robotics to large-scale industrial supply chains and logistics ~\cite{sutton2018rl,boute2022reinforcement,SALEH2026111693}. However, as these intelligent systems are increasingly deployed in cloud environments, they face stringent data privacy and confidentiality requirements. In scenarios such as smart logistics routing, transmitting sensitive operational data, like proprietary demand forecasts, financial data, or real-time vehicle GPS coordinates, to centralized servers presents critical security vulnerabilities~\cite{lin2017security,li2023multikey}. Recent work has shown that CKKS-based encryption can protect such location data in outsourced services such as ride-sharing and traffic management~\cite{xie2025secureloc}. Fully Homomorphic Encryption (FHE) offers a compelling solution to this privacy bottleneck by enabling computation directly on encrypted data~\cite{cheon2017ckks}. Under an FHE paradigm, a client can encrypt their operational state and transmit ciphertexts to a cloud-hosted neural network, which then computes and returns encrypted policy decisions without ever exposing the underlying raw data~\cite{gilad2016cryptonets,juvekar2018gazelle}. While we refer to this broad paradigm as FHE, our practical implementation utilizes Leveled Homomorphic Encryption (LHE). By exploiting the fixed depth of the neural network architecture, LHE avoids the prohibitive computational cost of ciphertext bootstrapping, enabling efficient real-time policy evaluation.

Despite its cryptographic guarantees, a fundamental limitation of FHE schemes such as CKKS is that they natively support only addition and multiplication~\cite{cheon2017ckks}. Consequently, standard non-linear activation functions like ReLU or sigmoid must be replaced with polynomial approximations~\cite{pulido2024slaf,araki2025training}. The standard approach utilizes degree-$d$ Chebyshev polynomials, which provide minimax approximation over a strictly bounded convergence domain $[-B, B]$~\cite{gilad2016cryptonets,araki2025training}. For a degree-2 polynomial evaluated on $[-5, 5]$:
\begin{equation}
P_2(z) = 0.469 + 0.5z + 0.094z^2 \approx \max(0, z)
\label{eq:poly}
\end{equation}

If the neural network's pre-activation functions move out-of-bounds $[-B, B]$, the higher-order polynomial terms cause catastrophic numerical explosion~\cite{araki2025training,lee2022lowcomplexity}. For example, at $z=50$, $P_2(50) = 260.47$ compared to $\text{ReLU}(50) = 50$. This failure mode is well-documented in FHE-based supervised learning, where solutions include boundary loss penalties~\cite{araki2025training}, trainable polynomial coefficients~\cite{pulido2024slaf}, and modified training methodologies~\cite{baruch2022hefriendly}.

None of the supervised-learning solutions addresses the unique dynamics of RL, and existing privacy-preserving RL frameworks~\cite{nguyen2025nmi,park2020pprl} do not target the polynomial activation divergence caused by recursive temporal-difference (TD) bootstrapping. In supervised learning, target labels are fixed, static constants. 
In RL, the TD targets are computed sequentially from the network's own outputs via bootstrapping~\cite{sutton2018rl}:

\begin{equation}
y_t = r_t + \gamma \max_a Q(s_{t+1}, a; \theta)
\label{eq:td_target}
\end{equation}

where $r_t$ is the transition reward, $\gamma \in [0, 1)$ is the discount factor, and $Q(s_{t+1}, a'; \theta)$ evaluates next-state actions under current network parameters $\theta$. Because each forward pass through the encrypted network injects a small polynomial approximation error, $\epsilon_{app}$. Because the target depends on the $Q$ values computed with the polynomial, the error accumulates recursively throughout training. The resulting steady-state value function error is bounded by $\|Q^* - \hat{Q}\|_\infty \leq \epsilon_{app}/(1-\gamma)$~\cite{bertsekas1996neuro,munos2008finite}. At a standard discount factor of $\gamma = 0.99$, this results in a $100\times$ amplification of the polynomial error. The error amplification creates a positive feedback loop---which we formally term the Bellman drift---that pushes pre-activation values out of the Chebyshev boundary, rendering standard FHE architectures numerically unstable.

To address this instability, the zero-mean advantage centering principle was applied directly to the temporal-difference bootstrapping targets while employing dueling architectures~\cite{wang2016dueling}. Note that the dueling architectures perform internal stream centering for network identifiability; however, they still output drift-prone $Q$-values.
Our approach tackles this shortcoming by introducing a holistic architectural framework for privacy-preserving RL. Our specific contributions are as follows:

\begin{enumerate}
\item We identify and formally define the Bellman drift as an RL-specific failure mode under FHE polynomial constraints, distinguishing it from the static approximation errors found in supervised learning literature.

\item We propose the Homomorphic Advantage Operator (HAO), a tripartite stabilization framework. It adapts the zero-mean advantage centering principle of Dueling architectures~\cite{wang2016dueling}, but shifts its application directly to the value recursion. Whereas dueling networks perform internal stream centering for network identifiability while still outputting full drift-prone $Q$-values, HAO applies a zero-mean linear projection matrix $P$ directly to the temporal-difference bootstrapping targets. This projection purges the uniform state-value baseline $V(s)$ from the recursive error loop, thus eliminating Bellman drift with no additional multiplicative depth. When paired with client-side reward scaling and decoupled weight decay, HAO strictly bounds pre-activations within degree-2 polynomial limits without requiring ciphertext bootstrapping.

\item We demonstrate through a three-tier experimental evaluation, including a tabular MDP, a CartPole environment using real TenSEAL CKKS cryptographic operations, and a single-depot capacitated routing benchmark, that the HAO framework achieves 0\% boundary breaches, whereas regularization alone breaches the bound on 3 of 5 random seeds and the unstabilized baseline in 83.8\% of episodes.
\item We show that the HAO framework remains stable when Gaussian noise of the kind used in DP-SGD is added to the clipped gradients ($\sigma$ up to 1.0), which makes it compatible with differentially private training.
\end{enumerate}

The remaining part of the paper is organized as follows. In Section~\ref{sec:related}, the related literature is reviewed. Section~\ref{sec:methodology} explains the proposed methodology in detail. Section~\ref{sec:results} evaluates the proposed framework numerically. In section~\ref{sec:discussion}, we discuss the main findings and key insights. Finally, Section~\ref{sec:conclusion} concludes the paper, including possible future extensions.

\section{Related Work}
\label{sec:related}

This section reviews the existing literature surrounding secure machine learning and cryptographic protocols. We categorize prior contributions into three distinct areas: techniques for managing polynomial activations under HE constraints, foundational cryptographic methodologies, and prior efforts in privacy-preserving RL. As summarized in \Cref{tab:related}, while extensive research addresses stability in supervised learning and static inference, the unique challenges of applying FHE to RL remain unexplored.

CryptoNets~\cite{gilad2016cryptonets} introduced the use of square activations ($x^2$) for HE-compatible inference. Subsequent work improved accuracy through self-learning activation functions (SLAF)~\cite{pulido2024slaf,xiong2024slafhe}, HE-friendly training with knowledge distillation~\cite{baruch2022hefriendly}, and boundary loss penalties with selective gradient clipping~\cite{araki2025training}. The ReBoot framework~\cite{pirillo2026reboot} enables fully encrypted training via depth-minimized architectures with CKKS bootstrapping.

Crucially, the previous stabilization methods are explicitly designed for supervised learning paradigms. Boundary loss penalties and SLAF rely on the assumption that target labels are static, ground-truth constants provided by a dataset. In RL, the agent generates its own dynamic targets via temporal-difference (TD) bootstrapping. The static training-time regularizers are insufficient to prevent the subsequent Bellman drift because the targets themselves are impacted by the polynomial approximation error.

The CKKS scheme~\cite{cheon2017ckks} supports approximate arithmetic on encrypted real numbers, serving as the foundation for our continuous state-space operations. Alternate approaches include GAZELLE~\cite{juvekar2018gazelle}, which combines HE linear layers with garbled circuits for exact non-linear operations, and Glyph~\cite{lou2020glyph}, which switches between CKKS and TFHE schemes to optimize different operation types. While mathematically robust, these mixed-protocol techniques incur significant communication and computational overhead, rendering them difficult to scale for highly sequential RL environments.

Prior work on privacy-preserving RL can be categorized into three distinct paradigms as follows:

(I) Several cryptographic data-protection approaches utilizing Multi-Party Computation (MPC) or Differential Privacy (DP) are proposed to protect distributed parameters. Namely, federated RL frequently deploys secure aggregation to mask client weight updates~\cite{bonawitz2017practical}, while DP-SGD bounds gradient privacy leakage~\cite{abadi2016deep}. Homomorphic encryption has likewise been used to protect gradients in distributed and federated deep learning~\cite{antwiboasiako2023cahe,gupta2025asyncfl}. However, these approaches do not provide a mechanism for an untrusted cloud server to compute policy forward passes directly on encrypted continuous state features.

(II) A growing body of work explores direct evaluation on encrypted data. For example, Park et al.~\cite{park2020pprl} proposed a privacy-preserving Q-learning framework using LWE-based FHE in cloud computing, validating feasibility on tabular problems but without neural-network function approximation or polynomial activations. Liu et al.~\cite{liu2019preyer} introduced the Preyer framework for dynamic treatment regimes, designing a plaintext length control protocol to manage bit-length overflow during repeated encrypted operations in small discrete state spaces. Jesu et al.~\cite{jesu2021rlencrypted} studied DQN training over AES-encrypted state inputs; notably, they evaluated homomorphic encryption (CKKS) but concluded that its computational latency was infeasible for iterative RL, falling back to symmetric block ciphers where performance collapsed on complex environments. Most recently, Nguyen et al.~\cite{nguyen2025nmi} introduced a groundbreaking privacy-preserving Soft Actor-Critic (SAC) framework using CKKS with full bootstrapping, achieving performance within 10\% of plaintext baselines on CartPole and Pendulum cases. Their core innovation is an HE-compatible Adam optimizer that reparameterizes momentum to bypass high-degree polynomial approximations of inverse square roots. However, their framework requires high-degree minimax polynomials (degrees 15 and 31) and full bootstrapping (incurring $\approx 505$\,s per training step on CartPole), and does not analyze whether pre-activations remain within the polynomial approximation domain.

(III) RL is increasingly applied to security problems themselves~\cite{javadpour2026beyond}, and recent studies investigate RL for FHE compilation rather than RL under FHE constraints. Sefsaf et al.~\cite{sefsaf2026chehab} utilize deep RL as a compiler optimization tool to automate FHE code vectorization, where the RL agent operates in plaintext. Similarly, Sun et al.~\cite{sun2026privacy} deploy FHE strictly for secure gradient aggregation in usage-based insurance, while policy evaluations remain in plaintext.

\Cref{tab:related} summarizes the surveyed solutions across supervised learning, cryptographic schemes, and privacy-preserving RL. While prior RL works address optimization bottlenecks~\cite{nguyen2025nmi} or tabular feasibility~\cite{park2020pprl}, none of them addresses the recursive polynomial activation explosion (Bellman drift) inherent to bootstrapped deep Q-learning. To the best of our knowledge, our work is the first to formalize this activation-domain failure mode and provide a zero-depth linear operator that enables stable deep RL under lightweight computations of degree-2 leveled HE, while avoiding costly bootstrapping evaluations.

\begin{table*}[t]
\centering
\caption{Taxonomy of existing solutions for polynomial activation and privacy-preserving learning under HE. PHE = Partially HE, LHE = Leveled HE (no bootstrapping), FHE = Fully HE (with bootstrapping).}
\label{tab:related}
\begin{tabular}{clllcp{6.0cm}}
\toprule
\# & Solution & Reference & HE & Paradigm & Mechanism Summary \\
\midrule
\multicolumn{6}{l}{\textit{Training-Time Stabilization (Supervised)}} \\
\midrule
1 & Boundary Loss & Araki et al.~\cite{araki2025training} & LHE & Supervised & Exponential penalty on $|z| > B$; prevents explosion during training \\
2 & Selective Grad.\ Clipping & Araki et al.~\cite{araki2025training} & LHE & Supervised & Clips gradients excluding BatchNorm $\gamma$, $\beta$; preserves normalization \\
3 & SLAF & Pulido-Gaytan \& Tchernykh~\cite{pulido2024slaf} & LHE & Supervised & Trainable polynomial coefficients; learns task-specific activations \\
4 & HE-Friendly Training & Baruch et al.~\cite{baruch2022hefriendly} & LHE & Supervised & Knowledge distillation + polynomial BN folding \\
5 & Weight Clamping / L2 & Standard technique & Any & Supervised & Constrains weight norms; bounds pre-activations indirectly \\
6 & Input Normalization & Standard technique & Any & Supervised & Scales inputs to $[-1, 1]$; keeps first-layer activations in domain \\
\midrule
\multicolumn{6}{l}{\textit{Cryptographic \& Architectural Techniques}} \\
\midrule
7 & CKKS Bootstrapping & Cheon et al.~\cite{cheon2017ckks} & FHE & Inference & Refreshes ciphertexts to reset noise budget; unlimited depth \\
8 & Leveled HE & BGV/BFV/CKKS & LHE & Inference & Pre-allocates depth budget; eliminates bootstrapping latency \\
9 & Scheme Switching & Lou \& Jiang~\cite{lou2020glyph} & LHE+ & Inference & Switches CKKS $\leftrightarrow$ TFHE for arithmetic vs.\ comparison ops \\
10 & CryptoNets ($x^2$) & Gilad-Bachrach et al.~\cite{gilad2016cryptonets} & LHE & Inference & Replaces ReLU with $x^2$; depth-1 activation \\
11 & ReBoot Framework & Pirillo \& Colombo~\cite{pirillo2026reboot} & FHE & Supervised & Depth-minimized encrypted training with CKKS bootstrapping \\
12 & Hybrid HE + MPC & Juvekar et al.~\cite{juvekar2018gazelle} & LHE & Inference & Linear layers in HE; non-linear activations in garbled circuits \\
\midrule
\multicolumn{6}{l}{\textit{Privacy-Preserving Reinforcement Learning}} \\
\midrule
13 & FHE Q-Learning & Park et al.~\cite{park2020pprl} & FHE & Tabular RL & LWE encrypted Q-table updates; no neural network or poly activations \\
14 & Preyer & Liu et al.~\cite{liu2019preyer} & Additive & Discrete RL & Encrypted experience replay with plaintext length overflow control \\
15 & State-Encrypted DQN & Jesu et al.~\cite{jesu2021rlencrypted} & AES & Deep RL & DQN on AES-encrypted states; abandoned HE due to latency \\
16 & HE-SAC & Nguyen et al.~\cite{nguyen2025nmi} & FHE & Deep RL & HE-compatible Adam optimizer; uses degree 15/31 poly with bootstrapping \\
\midrule
\multicolumn{6}{l}{\textit{This Work}} \\
\midrule
\textbf{17} & \textbf{HAO (This Work)} & \textbf{This paper} & \textbf{LHE} & \textbf{Deep RL} & \textbf{P-matrix centering eliminates $V(s)$ Bellman drift; degree-2, no ct--ct multiplication, no bootstrapping} \\
\bottomrule
\end{tabular}
\end{table*}
\section{Proposed Methodology}
\label{sec:methodology}

To overcome the theoretical limitations of standard regularization under FHE, we introduce a comprehensive architectural framework for privacy-preserving RL. First, this section formalizes the interaction between cryptographic mathematical constraints and the recursive nature of value-based learning; second, it introduces the Homomorphic Advantage Operator (HAO) to resolve these constraints, then proves its formal properties, and lastly outlines the client-server deployment architecture. \Cref{fig:general_arch} provides a high-level overview of the proposed framework.

\begin{figure*}[htbp]
\centering
\includegraphics[width=0.5\textwidth]{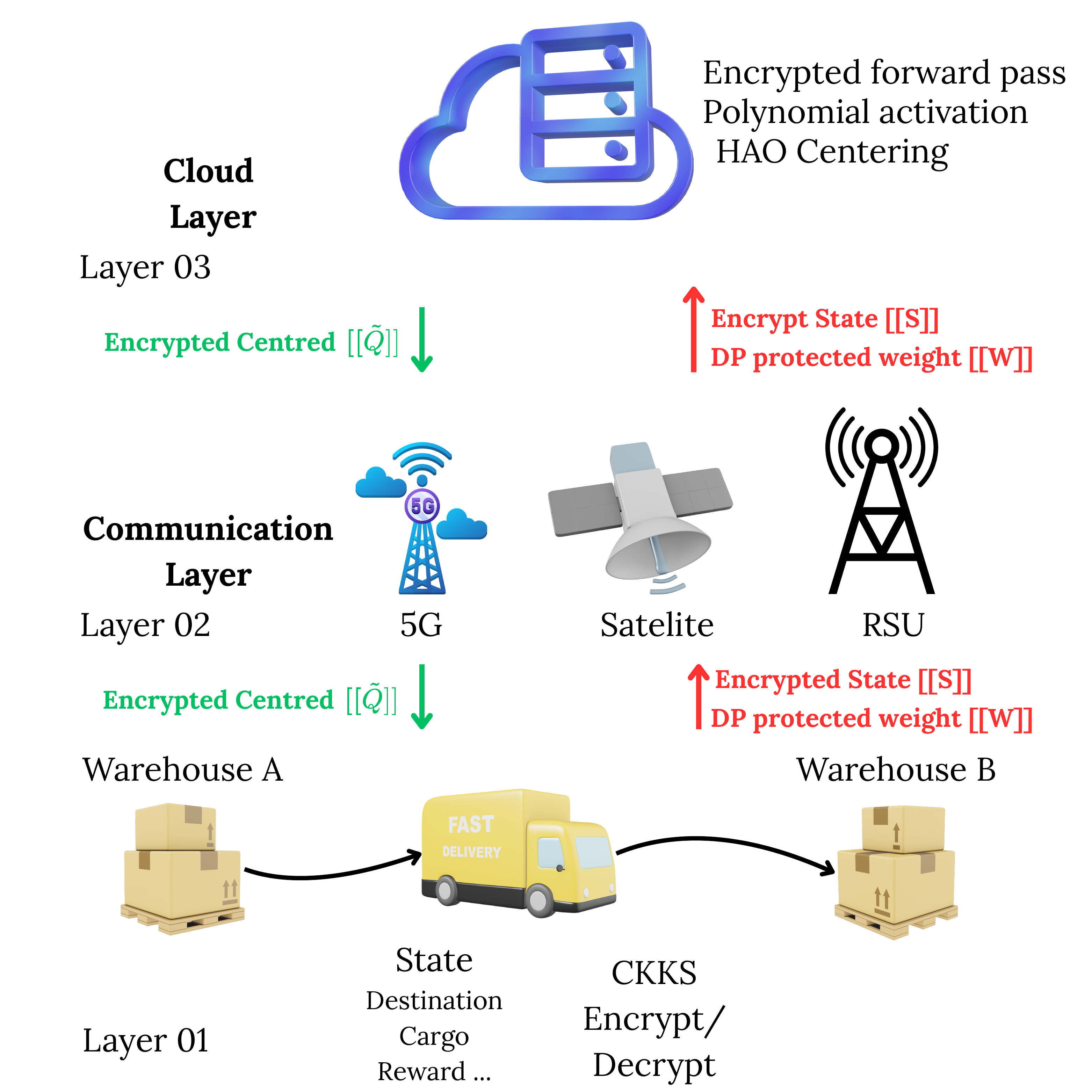}
\caption{High-level overview of the HAO framework for privacy-preserving reinforcement learning. The system integrates three interdependent components: (1)~CKKS-encrypted state transmission and polynomial forward evaluation, (2)~the HAO centering projection that neutralizes Bellman drift, and (3)~clipped weight updates with optional DP-SGD-style Gaussian noise.}
\label{fig:general_arch}
\end{figure*}

We define the privacy-preserving RL agent as a deep Q-network mapping environmental states to a vector of action values. To ensure compatibility with the HE scheme, the network is restricted to a degree-\textit{d} polynomial activation function $P_d$. For the mechanical constraint, we model the polynomial bottleneck at the primary hidden layer, where the forward pass is defined as:
\begin{equation}
Q(s, \cdot; \theta) = P_d(sW_1 + b_1) W_2 + b_2
\end{equation}
where $\theta = \{W_1, b_1, W_2, b_2\}$ denotes the network parameters, $s$ is the environmental state, and the dot ($\cdot$) indicates that $Q(s, \cdot; \theta) \in \mathbb{R}^{|\mathcal{A}|}$ represents the vector of Q-values evaluated across all possible actions in the action space $\mathcal{A}$ with cardinality $|\mathcal{A}|$. Additionally, $P_d$ is a degree-$d$ Chebyshev approximation of the ReLU function. This approximation is mathematically valid only within a strictly bounded domain $[-B, B]$. Consequently, the pre-activation vector $z_1 = sW_1 + b_1$ must satisfy the hard constraint $\|z_1\|_\infty \leq B$, $\| \cdot \|_\infty$ is the supremum norm, to prevent the polynomial divergence.

\begin{theorem}[Bellman Drift Bound]
For a $\gamma$-discounted MDP with a per-step polynomial approximation error bounded by $\epsilon_{app}$, the steady-state value function error satisfies:
\begin{equation}
\|Q^* - \hat{Q}\|_\infty \leq \frac{\epsilon_{app}}{1-\gamma}
\end{equation}
\end{theorem}

where $Q^*$ represents the optimal action-value function, $\hat{Q}$ denotes the approximated value function under the polynomial evaluation, and $\| \cdot \|_\infty$ is the supremum norm. This theoretical upper bound adapts the finite-time analysis of approximate value iteration established by Munos and Szepesv\'{a}ri~\cite{munos2008finite}, proving that even marginal polynomial deviations are severely amplified by the recursive nature of the Bellman operator.

To understand why traditional neural network regularization fails to prevent this amplification, we must examine the internal structure of the action values. Any Q-value can be structurally decomposed as $Q(s,a) = V(s) + A(s,a)$, where $V(s) = \frac{1}{|\mathcal{A}|}\sum_{a'} Q(s,a')$ is the state-value baseline and $A(s,a) = Q(s,a) - V(s)$ represents the relative advantage of an action~\cite{wang2016dueling}. In practical environments, the magnitude of the state-value baseline typically dwarfs the advantage function, meaning $|V(s)| \gg |A(s,a)|$.

Standard stabilization techniques such as L2 weight decay and gradient clipping successfully control overall weight magnitude but fail to eliminate the persistent directional pressure exerted by $V(s)$. Every TD update systematically pulls the network weights toward this large uniform baseline. This asymmetric force inevitably overpowers symmetric weight decay. The vulnerability is particularly pronounced in environments with dense continuous state spaces, such as logistics routing. Because all input features activate simultaneously, they rapidly accelerate pre-activation growth, virtually guaranteeing a catastrophic breach of the boundary $B$.

\subsection{The Homomorphic Advantage Operator (HAO)}

To overcome the theoretical limitations of standard regularization, we introduce the Homomorphic Advantage Operator (HAO). This operator structurally eliminates the persistent Bellman drift without introducing prohibitive computational overhead to the HE circuit. The method integrates a zero-depth centering projection with symmetric regularization.

The core insight is that subtracting the state-value baseline neutralizes the persistent directional drift that drives pre-activations outside the polynomial domain $[-B, B]$. Crucially, this centering is achieved via a constant linear projection matrix multiplication, which is natively supported by the CKKS scheme and adds zero non-linear multiplicative depth.

Formally, we define the centering matrix:
\begin{equation}
P = I_{|\mathcal{A}|} - \frac{1}{|\mathcal{A}|}\mathbf{1}\mathbf{1}^T
\label{eq:pmatrix}
\end{equation}
where $\mathcal{A}$ denotes the discrete action space, $I_{|\mathcal{A}|}$ is the identity matrix of size $|\mathcal{A}| \times |\mathcal{A}|$, $\mathbf{1}$ is the all-ones row vector of length $|\mathcal{A}|$, and $\mathbf{1}^T$ is its transpose.

For any state $s$, the HAO modifies the TD target by projecting the Q-values through $P$:
\begin{equation}
\tilde{Q}(s, \cdot) = Q(s, \cdot) \cdot P = A(s, \cdot)
\label{eq:hao}
\end{equation}
where $A(s, a) = Q(s, a) - \frac{1}{|\mathcal{A}|}\sum_{a'} Q(s, a')$ is the advantage function. Applying this centering projection to the bootstrapping evaluation at the successor state $s_{t+1}$, the modified TD target becomes:
\begin{equation}
y_t = r_t + \gamma \max_a \tilde{Q}(s_{t+1}, a)
\end{equation}

It is essential to distinguish the HAO projection from the advantage formulation in Dueling DQN~\cite{wang2016dueling}. Dueling networks decompose the internal neural network architecture into separate value and advantage streams, recombining them as $Q(s,a) = V(s) + (A(s,a) - \frac{1}{|\mathcal{A}|}\sum_{a'} A(s,a'))$. While this achieves parameter identifiability, the final output $Q(s,a)$ still retains the large, drifting $V(s)$ baseline, which amplifies the subsequent TD targets error. Moreover, implementing dual-stream branching architectures under FHE incurs significant cryptographic memory and routing overhead. In contrast, HAO operates on a standard single-stream network and applies the centering projection $P$ directly to the TD target computation ($y_t = r_t + \gamma \max_a \tilde{Q}(s_{t+1},a)$). By feeding strictly de-meaned advantage values into the recursive backup, HAO completely severs the positive feedback loop of the Bellman drift while consuming zero non-linear multiplicative depth.

While the centering projection mathematically secures the Bellman contraction, empirical stability in dense environments requires integrating it with scalar bounds and symmetric regularization. The full HAO framework combines three interdependent components: (1)~HAO Centering (Eq.~\ref{eq:hao}), which removes the $V(s)$ directional bias; (2)~Client-side Reward Scaling ($\alpha$), which proportionally compresses the active advantage stream $|A(s,a)|$, keeping the centered targets within polynomial bounds; and 
(3)~Decoupled Weight Decay, which applies differential regularization across layers ($\lambda_1 > \lambda_2$) alongside gradient clipping threshold $G$ and weight decay $\lambda$. This is done to bound the per-step update magnitudes and prevent long-term parameter growth. 

Our ablation study in \Cref{sec:results} demonstrates that all three are necessary: centering without regularization diverges (no gradient control), and regularization without centering violates the bound $B$ therefore it cannot counteract $V(s)$ bias.

\subsection{Formal Properties and the Operator Distortion Identity}
\label{sec:theory}
Because the HAO modifies the Bellman target, we formalize its theoretical properties by analyzing both static policy evaluation and dynamic temporal-difference bootstrapping.
\begin{proposition}[HE-Compatibility]
\label{prop:he}
The HAO centering requires only a single matrix--vector multiplication $\tilde{Q} = Q \cdot P$, which is natively supported by the CKKS scheme. It adds zero additional non-linear multiplicative depth to the HE circuit.
\end{proposition}
\begin{proof}
The matrix $P$ is constant and precomputed. The operation $Q \cdot P$ decomposes into CKKS-native additions and multiplications by plaintext constants. No comparison, division, or branching operations are required. It requires no ciphertext--ciphertext multiplication; the plaintext--ciphertext product consumes at most one rescaling level, which fits within the leveled CKKS parameters used here, so no bootstrapping is needed.
\end{proof}
\begin{proposition}[Per-State Action Ranking Preservation]
\label{prop:policy}
For any fixed state $s$ and any action-value vector $Q(s, \cdot) \in \mathbb{R}^{|\mathcal{A}|}$, the HAO centering preserves greedy action selection:
\begin{equation}
\arg\max_a \tilde{Q}(s, a) = \arg\max_a Q(s, a).
\end{equation}
\end{proposition}
\begin{proof}
For any state $s$, the centered action values are given by $\tilde{Q}(s,a) = Q(s,a) - \bar{Q}(s)$, where $\bar{Q}(s) = \frac{1}{|\mathcal{A}|}\sum_{a'} Q(s,a') \equiv V(s)$ denotes the mean action-value (the state-value baseline). Because $\bar{Q}(s)$ is independent of the action $a$, subtracting it is a uniform shift across all action values for state $s$:

\begin{equation}
\begin{aligned}
\arg\max_a \tilde{Q}(s,a) &= \arg\max_a \bigl[Q(s,a) - \bar{Q}(s)\bigr] \\
&= \arg\max_a Q(s,a).
\end{aligned}
\end{equation}

Thus, greedy policy extraction at any evaluated state is invariant under the HAO projection.
\end{proof}
\begin{theorem}[Span Contraction of Ideal Row Centering]
\label{thm:contraction}
Let $\mathcal{T}$ denote the standard Bellman optimality operator, and let $\mathcal{T}_{\text{row}} Q(s, \cdot) = (\mathcal{T} Q)(s, \cdot) \cdot P$ denote the idealized row-centered Bellman operator across all actions at state $s$. Then $\mathcal{T}_{\text{row}}$ is isometric in the span seminorm ($sp$), preserving the $\gamma$-contraction:
\begin{equation}
sp(\mathcal{T}_{\text{row}} Q_1 - \mathcal{T}_{\text{row}} Q_2) \leq \gamma \, sp(Q_1 - Q_2).
\end{equation}
\end{theorem}
\begin{proof}
For any state $s$, projecting by $P$ subtracts the uniform scalar mean $\bar{y}(s) = \frac{1}{|\mathcal{A}|}\sum_{a'} (\mathcal{T} Q)(s, a')$ from all action components at $s$. For any vector $v$, $sp(v \cdot P) = \max_i (v_i - \bar{v}) - \min_i (v_i - \bar{v}) = \max_i v_i - \min_i v_i = sp(v)$. Because $\mathcal{T}$ is a $\gamma$-contraction in the supremum norm, its composition with the span isometry $P$ strictly satisfies $sp(\mathcal{T}_{\text{row}} Q_1 - \mathcal{T}_{\text{row}} Q_2) \leq \gamma \, sp(Q_1 - Q_2)$.
\end{proof}
In sample-based model-free deep RL, transition samples $(s, a, r, s')$ provide access only to the observed successor state $s' = \text{succ}(s, a)$, leading to the sampled HAO target:
\begin{equation}
y_t(s,a) = r(s,a) + \gamma \max_{a'} \tilde{Q}(s', a') = (\mathcal{T}_{\text{succ}} Q)(s, a).
\end{equation}
We now establish the exact algebraic relationship between the sample-based realization $\mathcal{T}_{\text{succ}}$ and the ideal operator $\mathcal{T}_{\text{row}}$.
\begin{theorem}[Operator Distortion Identity]
\label{thm:distortion}
For any state $s$ and any two actions $a, a' \in \mathcal{A}$ with successor states $s_a' = \text{succ}(s,a)$ and $s_{a'}' = \text{succ}(s,a')$, the difference between the sample-based targets satisfies the exact identity:
\begin{equation}
\mathcal{T}_{\text{succ}} Q(s, a) - \mathcal{T}_{\text{succ}} Q(s, a') = D(a, a') - E(a, a')
\end{equation}
where $D(a, a') = \mathcal{T}_{\text{row}} Q(s, a) - \mathcal{T}_{\text{row}} Q(s, a')$ is the true relative advantage, and $E(a, a') = \gamma [\bar{Q}(s_a') - \bar{Q}(s_{a'}')]$ is the successor baseline distortion.
\end{theorem}
\begin{proof}
By definition of $\mathcal{T}_{\text{succ}}$:
\begin{align*}
\mathcal{T}_{\text{succ}} Q(s, a) &= r(s,a) + \gamma \max_b Q(s_a', b) - \gamma \bar{Q}(s_a') \\
&= (\mathcal{T} Q)(s, a) - \gamma \bar{Q}(s_a').
\end{align*}
Subtracting the target for action $a'$:
\begin{align*}
\mathcal{T}_{\text{succ}} Q(s, a) - \mathcal{T}_{\text{succ}} Q(s, a') &= \\ [(\mathcal{T} Q)(s, a) - (\mathcal{T} Q)(s, a')]  - \gamma [\bar{Q}(s_a') - \bar{Q}(s_{a'}')].
\end{align*}
Because row-centering at state $s$ subtracts the identical scalar $\bar{y}(s)$ from all actions at $s$, $(\mathcal{T} Q)(s, a) - (\mathcal{T} Q)(s, a') \equiv \mathcal{T}_{\text{row}} Q(s, a) - \mathcal{T}_{\text{row}} Q(s, a') = D(a, a')$. Setting $E(a, a') = \gamma [\bar{Q}(s_a') - \bar{Q}(s_{a'}')]$ completes the identity.
\end{proof}

\begin{corollary}[Action-Selection Preservation and Scale Invariance]
\label{cor:ranking}
Let $a^* = \arg\max_{a \in \mathcal{A}} \mathcal{T}_{\text{row}} Q(s, a)$ denote the greedy optimal action under the ideal operator. The sample-based target preserves the optimal greedy selection $\arg\max_{a \in \mathcal{A}} \mathcal{T}_{\text{succ}} Q(s, a) = a^*$ if and only if:
\begin{equation}
\forall a \neq a^*, \quad E(a^*, a) < D(a^*, a).
\label{eq:ranking_criterion}
\end{equation}
Furthermore, this criterion is strictly scale-invariant: for any positive linear reward scaling factor $\alpha > 0$, $\text{sign}(\alpha D - \alpha E) \equiv \text{sign}(D - E)$.
\end{corollary}

Corollary~\ref{cor:ranking} delineates the exact operational mechanisms of the HAO framework:
\begin{enumerate}
\item \textbf{Theoretical Foundation ($\mathcal{T}_{\text{row}}$ vs.\ $\mathcal{T}_{\text{succ}}$):} Theorem~\ref{thm:contraction} guarantees non-expansive span contraction for ideal row centering. In sample-based deep RL, $\mathcal{T}_{\text{succ}}$ serves as a computationally lightweight surrogate whose stability transfers conditionally upon satisfying \eqref{eq:ranking_criterion}.
\item \textbf{Reward Scaling ($\alpha$):} As proven by Corollary~\ref{cor:ranking}, $\alpha$ is scale-invariant with respect to ordinal action rankings and cannot alter the sign of $D - E$. Its sole role is to bound pre-activation magnitude ($\|z\|_\infty \leq B$), preventing catastrophic polynomial domain escape.
\item \textbf{Empirical Distortion Margin:} At decision-critical bifurcation states (e.g., balance vs.\ terminal failure in CartPole), $D(a^*, a)$ and $E(a^*, a)$ track each other on the same order of magnitude. Whether $D(a^*, a) > E(a^*, a)$ holds is an empirical property determined by the environment's reward margins and downstream advantage spreads, which can be directly monitored via the distortion ratio $\rho(s) = \max_{a \neq a^*} \frac{E(a^*, a)}{D(a^*, a)} < 1$.
\end{enumerate}

\subsection{Client-Server Architecture and Privacy}

We consider a logistics company (client) outsourcing RL computation to a cloud service (server), see \Cref{fig:architecture} for a full illustration.

\begin{figure*}[t]
\centering
\resizebox{\textwidth}{!}{\input{architecture}}
\caption{HAO client-server architecture for privacy-preserving RL. The client encrypts private states via CKKS, the server computes an encrypted forward pass with polynomial activations, and the client applies HAO centering after decryption. Weight updates can be perturbed with DP-SGD-style Gaussian noise. The HAO centering (gold) is the core contribution: it removes the $V(s)$ Bellman drift without any ciphertext--ciphertext multiplication.}
\label{fig:architecture}
\end{figure*}
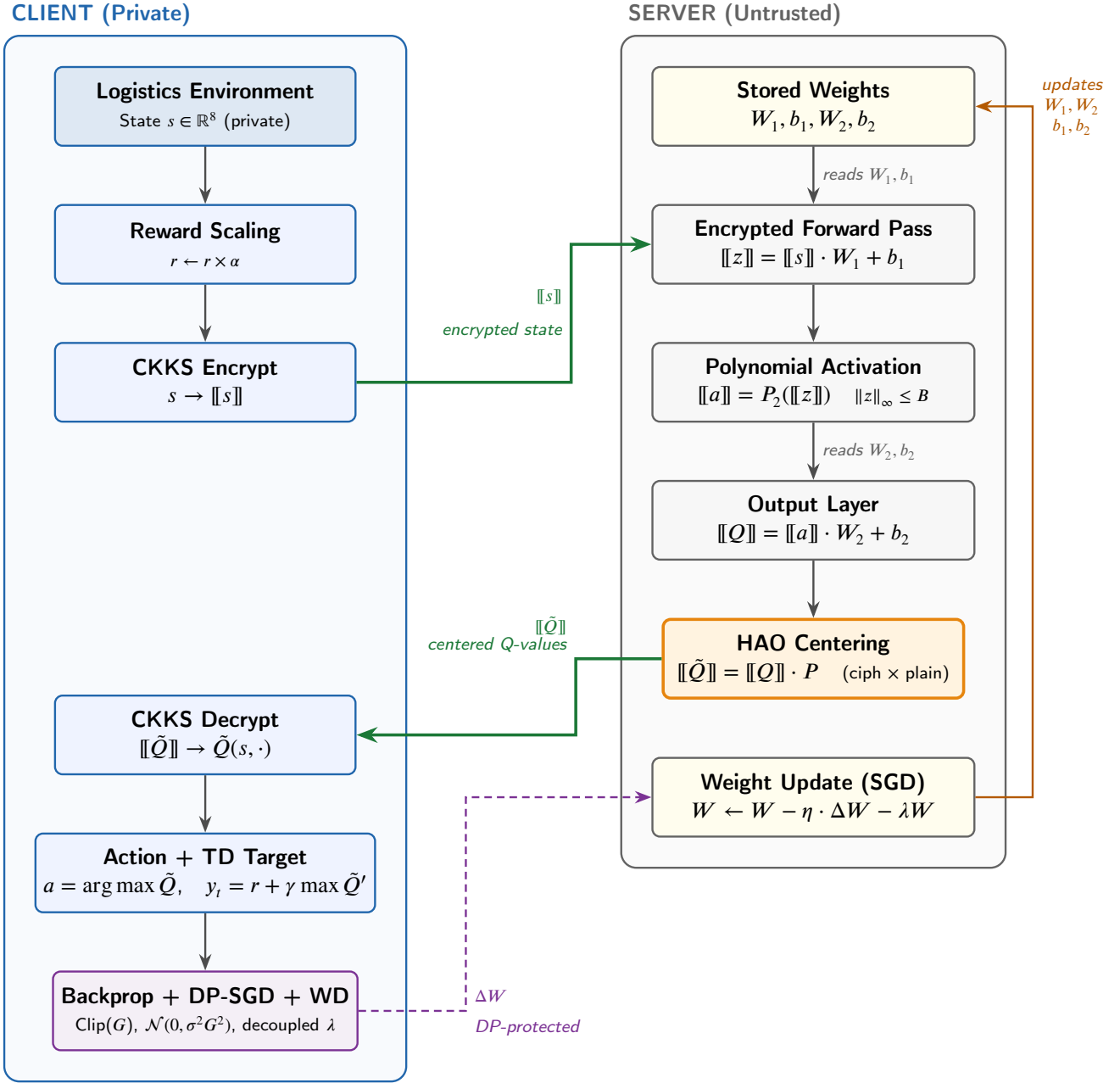

We adopt the standard honest but curious threat model used throughout the HE-ML literature~\cite{gilad2016cryptonets,juvekar2018gazelle,pirillo2026reboot}. According to this model, the server faithfully executes the protocol but may attempt to infer private information from the observed data. The primary privacy goal is data confidentiality, where the server must not learn the client's raw states, Q-values, rewards, and actions.

The protocol operates as follows:
\begin{itemize}
\item \textbf{Forward Pass (Encrypted)}: The client encrypts the state $s \rightarrow [[s]]$ using CKKS and sends it to the server. The server computes the encrypted forward pass: $[[Q]] = P_2([[s]] \cdot W_1 + b_1) \cdot W_2 + b_2$, and returns $[[Q]]$ to the client. This is a Ciphertext-Plaintext (CxP) protocol where model weights remain in plaintext on the server.
\item \textbf{Training Update}: The $P$-matrix centering ($\tilde{Q} = Q \cdot P$) can be applied either by the server before decryption as a CxP operation to support thin-client deployments, or by the client after the PxP decryption operation to minimize server operations. The client then decrypts the values if not already plaintext, computes TD targets and gradients in plaintext, applies clipping and weight decay, and sends weight updates $\Delta W$ to the server.
\end{itemize}

The weight updates $\Delta W$ transmitted to the server constitute a potential privacy leakage channel, as gradient inversion attacks can reconstruct training data from observed gradients~\cite{zhu2019dlg}. It is common in privacy-preserving RL to combine FHE with DP. Therefore, the proposed HAO framework is also designed to seamlessly integrate with DP mechanisms.

We address gradient leakage with the Gaussian mechanism of DP-SGD~\cite{abadi2016deep}. In DP-SGD, the client adds calibrated noise to clipped gradients before data transmission. The gradient clipping bound $G$ from the proposed HAO framework serves a dual purpose by stabilizing the Bellman drift and bounding the L2 sensitivity required by the Gaussian mechanism:
\begin{equation}
\Delta\tilde{W} = \text{clip}(\Delta W, G) + \mathcal{N}(0, \sigma^2 G^2 I)
\end{equation}
Since DP noise is zero-mean, it does not create directional bias. HAO's centering addresses the orthogonal problem of systematic Bellman drift. HAO's gradient clipping bound $G$ directly provides the L2 sensitivity parameter, making HAO DP-ready by construction. Our experiments (\Cref{sec:results}) demonstrate that HAO maintains 0\% boundary breaches even under strong noise ($\sigma=1.0$). In these experiments, the noise is added to the clipped mini-batch gradient; a formal $(\epsilon, \delta)$ guarantee additionally requires per-example clipping and a privacy accountant, which we leave to future work.

\section{Results}
\label{sec:results}

To rigorously evaluate the HAO framework, we conduct a progressive series of experiments spanning isolated mechanics to high-dimensional cryptographic deployment. The results uniformly demonstrate that the $P$-matrix projection successfully bounds the Bellman drift while preserving the optimal agent policy. The environments, training parameters, and evaluated modes for these experiments are summarized in \Cref{tab:setup}.

\begin{table*}[t]
\centering
\caption{Summary of Experimental Setup Methodology.}
\label{tab:setup}
\begin{tabular}{l p{4.0cm} p{2.5cm} p{3.5cm} p{4cm}}
\toprule
Experiment & Environment & Architecture & Training Parameters & Evaluated Modes \\
\midrule
1: Tabular MDP & 10 states, 4 actions & 1 Hidden Layer (244 params) & 2000 epochs, 5 random seeds & Standard, Full HAO, Centering Only \\
2: Encrypted CartPole & Hybrid client-server using real TenSEAL CKKS operations & 1 Hidden Layer (114 params) & 250 episodes & Standard vs.\ HAO LHE \\
3: Logistics Ablation & 20-node routing, 8D continuous states, 6 actions, $\alpha=0.05$ & 1 Hidden Layer (486 params) & 1000 eps, 5 seeds, shift at ep 500 & Full HAO, Centering Only, Reg Only, No Stab., HAO+DP \\
\bottomrule
\end{tabular}
\vspace{2ex}
{\par}
{Across all experiments, we calibrate the reward scale $\alpha=0.05$ such that Q-value magnitudes stress-test the polynomial domain boundary. This mirrors standard reward scaling practices in deep RL~\cite{mnih2015humanlevel,vanhasselt2021returnscaling}. Total trainable parameters per architecture (including input/output layers): Exp 1 (244 params), Exp 2 (114 params), Exp 3 (486 params).
\par}
\end{table*}

We isolate the mechanical effect of the $P$-matrix projection using the tabular MDP environment. As shown in \Cref{tab:exp1} and \Cref{fig:exp1}, the standard Q-learning formulation without stabilization immediately exhibits polynomial inflation, driving pre-activations up to $\|z\|_\infty = 1.83$ (compared to $0.74$ under HAO). In unconstrained environments, this super-linear growth rapidly escalates toward the Chebyshev boundary $B=5.0$, where the steep derivative of the polynomial activation function amplifies the approximation error. As illustrated in \Cref{fig:exp1_grad}, this uncentered baseline exhibits exploding gradient norms of $5.04 \pm 0.34$. This instability directly impairs learning, degrading policy decision accuracy to only $68.0\% \pm 8.4\%$.

By contrast, both the Centering Only and Full HAO modes successfully isolate and remove the state-value baseline. This reduces the maximum pre-activations by $2.5\times$ (down to $0.70$ and $0.74$) and compresses gradient norms by nearly two orders of magnitude ($0.06 \pm 0.00$). Crucially, as reported in \Cref{tab:exp1}, this numerical stabilization directly translates into superior decision-making. Both centered modes achieve $86.0\% \pm 5.5\%$ policy accuracy, representing an $18.0\%$ absolute improvement over standard Q-learning.

\begin{figure}[h]
\centering
\includegraphics[width=1.0\columnwidth]{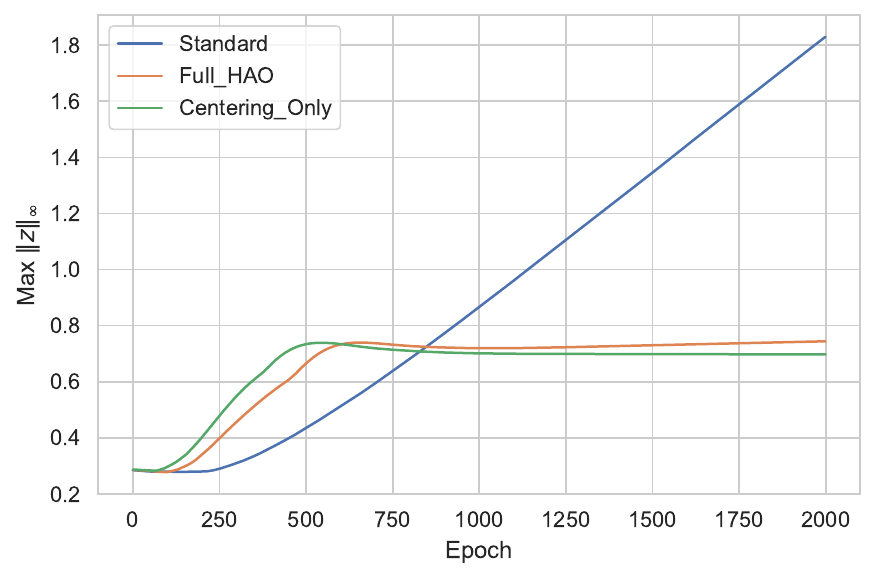}
\caption{Pre-activation growth over 2000 epochs. Centering successfully bounds polynomial inflation and stabilizes training.}
\label{fig:exp1}
\end{figure}

\begin{figure}[h]
\centering
\includegraphics[width=1.0\columnwidth]{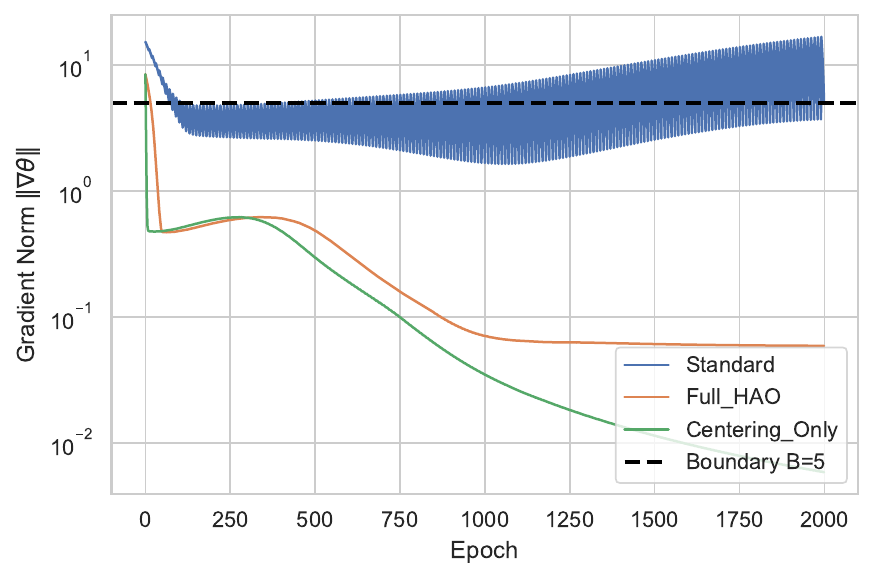}
\caption{Gradient norm dynamics on a logarithmic scale over 2000 epochs. The uncentered baseline exhibits exponential gradient explosion upon breaching the polynomial boundary, whereas HAO centering compresses gradient variance.}
\label{fig:exp1_grad}
\end{figure}

\begin{table}[h]
\centering
\caption{Tabular MDP (2000 epochs, 5 seeds)}
\label{tab:exp1}
\begin{tabular}{lccc}
\toprule
Mode & Accuracy (\%) & Max $\|z\|_\infty$ & Grad Norm \\
\midrule
Standard & $68.0 \pm$ \textbf{8.4} & $1.83 \pm 0.20$ & $5.04 \pm 0.34$ \\
Full HAO & \textbf{86.0} $\pm$ 5.5 & \textbf{0.74} $\pm$ \textbf{0.07} & \textbf{0.06} $\pm$ \textbf{0.00} \\
Centering Only & \textbf{86.0} $\pm$ 5.5 & \textbf{0.70} $\pm$ \textbf{0.08} & \textbf{0.01} $\pm$ \textbf{0.00} \\
\bottomrule
\end{tabular}
\end{table}

To observe the theoretical Bellman drift under actual cryptographic constraints, we deployed a hybrid client-server CartPole agent that utilizes TenSEAL CKKS operations for the forward pass. All experiments were executed using standard Kaggle CPU environments, demonstrating the feasibility of the protocol without requiring specialized hardware accelerators.

\begin{figure}[h]
\centering
\includegraphics[width=1.0\columnwidth]{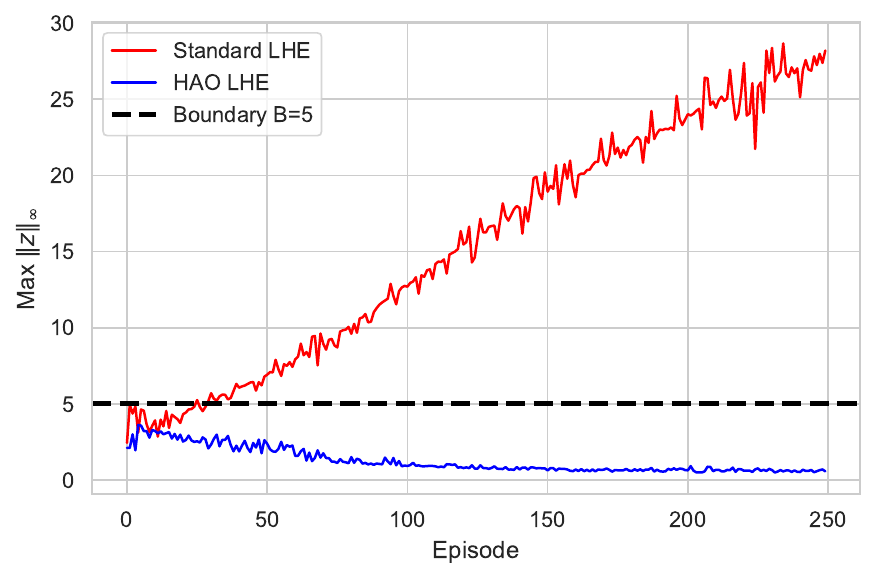}
\caption{Trapping region (pre-activation trajectory) for the hybrid TenSEAL CKKS CartPole agent. The network utilizes a single hidden layer architecture.}
\label{fig:exp2_preact}
\end{figure}
\begin{figure}[h]
\centering
\includegraphics[width=0.7\columnwidth]{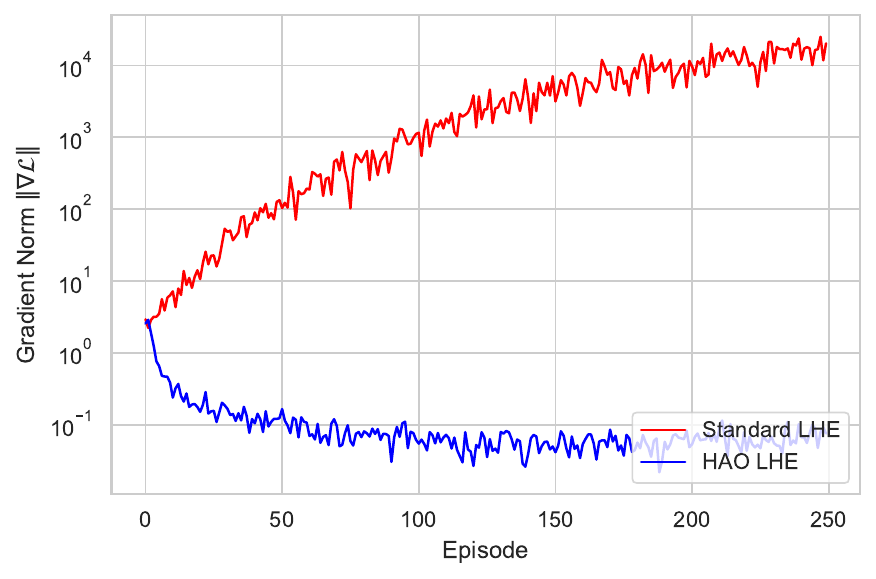}
\caption{Optimization health (gradient norm) over 250 episodes (log scale). Standard LHE diverges exponentially.}
\label{fig:exp2_grad}
\end{figure}

As shown in \Cref{tab:exp2}, \Cref{fig:exp2_preact}, and \Cref{fig:exp2_grad}, under real CKKS arithmetic using a single hidden layer neural network, the standard Leveled HE (LHE) agent completely collapses. The pre-activations grow monotonically from an initial value of 2.44 to an explosive 28.68---a catastrophic breach of $5.7\times B$. This super-linear inflation directly triggers the polynomial approximation error loop formalized in Theorem 1, pushing gradient norms to 25,003 and rendering the agent numerically unstable, capping out at a maximum reward of only 28.0. Conversely, the pre-activations of the HAO-stabilized LHE agent never exceed 3.61, remaining inside $[-B, B]$ throughout training, and settle at $\|z\|_\infty = 0.57$ by the final episode. This stabilizes the gradients and allows the agent to learn a functional policy (Max.~reward 93.0), while the encrypted centering step ($[[Q]] \cdot P$) takes 1.64 seconds per episode on average. Crucially, in this experiment, the centering projection is applied server-side as a CxP operation ($[[Q]] \cdot P$), demonstrating that HAO is equally effective when performed directly on encrypted data.
\begin{table}[h]
\centering
\caption{CKKS Encrypted CartPole (250 episodes). Q$\cdot$P time is the mean time per episode of the encrypted centering step.}
\label{tab:exp2}
\begin{tabular}{lcccc}
\toprule
Mode & Max $\|z\|_\infty$ & Max $\|\nabla \theta\|$ & Max. & Avg. \\
& & & Reward & Q$\cdot$P time (s)\\
\midrule
Standard LHE & $28.68$ & $25,003$ & $28.0$ & N/A \\
HAO LHE & \textbf{3.61} & \textbf{2.88} & \textbf{93.0} & \textbf{1.64} \\
\bottomrule
\end{tabular}
\end{table}

To further evaluate HAO on a realistic, high-dimensional industrial task, we deploy a 20-node, single-depot, single-vehicle capacitated routing environment inspired by the Capacitated Dynamic Vehicle Routing Problem (CDVRP) formulation in RL literature~\cite{boute2022reinforcement,nazari2018reinforcement}.
Unlike the low-dimensional or sparse states of the tabular MDP and CartPole, this benchmark stresses polynomial activation bounds via an 8-dimensional dense continuous state space (encoding normalized vehicle coordinates, current and remaining customer demands, vehicle capacity utilization, visitation progress, time pressure, and distance to the nearest unvisited node) paired with a 6-action candidate dispatch space. Because all continuous input features activate simultaneously, they accelerate pre-activation accumulation. To rigorously test the robustness of the stabilization mechanisms, we inject a severe non-stationary environment shift at Episode 500. This shift scrambles the delivery demands and perturbs all node locations, delivering a sudden shock to the TD-targets. This forces the agent to rapidly adapt to a novel state distribution, thereby stress-testing the network's susceptibility to catastrophic gradient explosion under sudden temporal difference errors. As in the tabular experiment, this ablation trains the FHE-compatible network (degree-2 activation) in plaintext; the encrypted forward pass is evaluated in the CartPole experiment.

\begin{figure}[h]
\centering
\includegraphics[width=1.0\columnwidth]{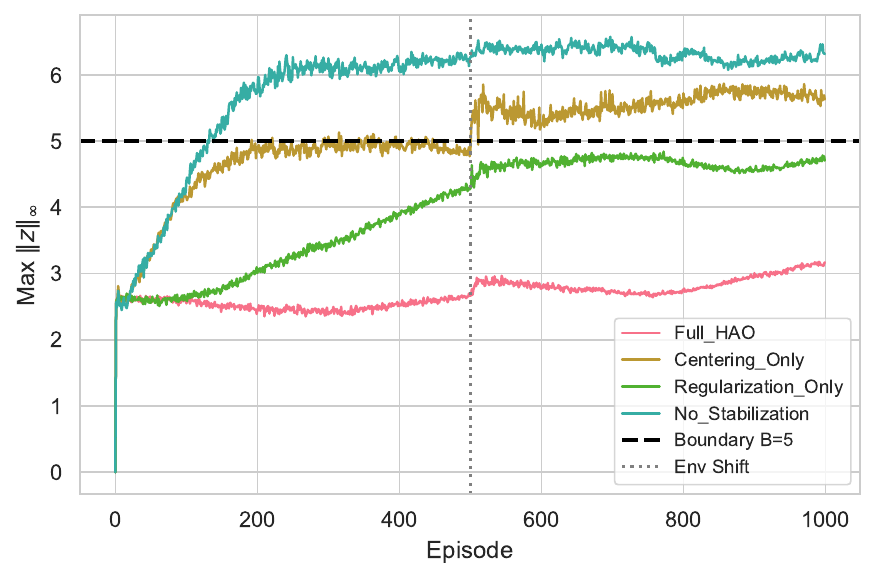}
\caption{Pre-activation growth over 1,000 episodes. The HAO framework strictly bounds pre-activations below $B=5.0$, whereas uncentered baseline agents diverge under the environment shift.}
\label{fig:exp3_preact}
\end{figure}

\begin{figure}[h]
\centering
\includegraphics[width=1.0\columnwidth]{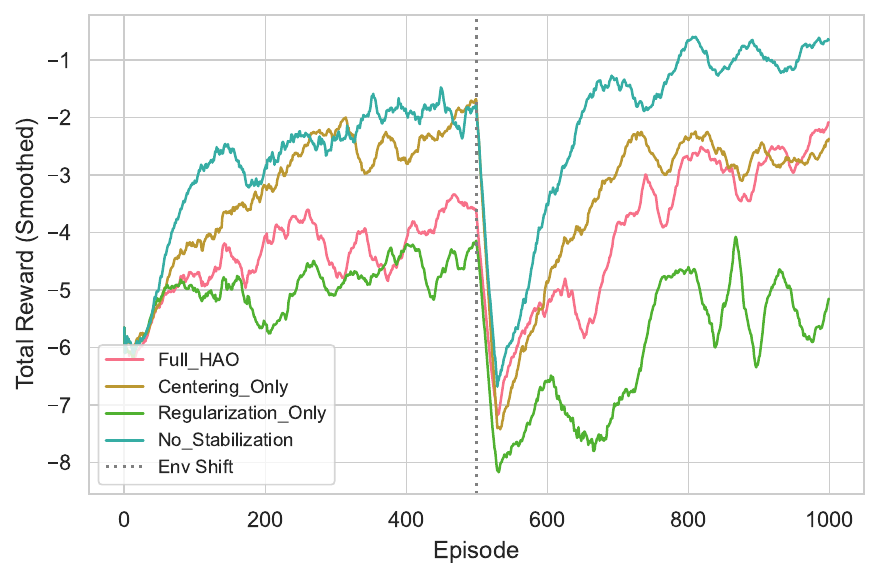}
\caption{Agent performance (moving average reward) over 1,000 episodes in the 20-node Logistics environment. Full HAO maintains optimal performance post-shift.}
\label{fig:exp3_reward}
\end{figure}

\begin{table}[h]
\centering
\caption{Realistic Logistics ($\gamma=0.99$, $\alpha=0.05$, 5 seeds)}
\label{tab:exp3}
\begin{tabular}{lccc}
\toprule
Mode & Max $\|z\|_\infty$ & Breach \% & FHE-Deploy? \\
\midrule
Full HAO & \textbf{3.89} & \textbf{0.0\% }& \textbf{Yes} \\
Centering Only & 7.79 & 54.5\% & No \\
Regularization Only & 6.98 & 23.8\% & No \\
No Stabilization & 7.82 & 83.8\% & No \\
\midrule
HAO + DP Moderate & \textbf{4.12} & \textbf{0.0\%} & \textbf{Yes} \\
HAO + DP Strong & \textbf{4.01} & \textbf{0.0\%} & \textbf{Yes} \\
\bottomrule
\end{tabular}
\end{table}

The ablation results (\Cref{tab:exp3}, \Cref{fig:exp3_preact}, and \Cref{fig:exp3_reward}) clearly demonstrate that standard neural network stabilization techniques are fundamentally insufficient. The 'Regularization Only' mode, which relies exclusively on L2 weight decay and gradient clipping, breaches the boundary on 3 of 5 random seeds (23.8\% of all episodes). This occurs because the asymmetric force of the drifting state-value baseline $V(s)$ systematically overpowers the symmetric pull of standard weight decay.

The Full HAO framework, however, achieves a 0\% boundary breach rate across all 5,000 evaluation episodes. By utilizing the $P$-matrix projection to annihilate the directional $V(s)$ bias, the clipping and decay mechanisms are finally able to function symmetrically. This confirms that the centering operation is not merely a helpful heuristic, but a strict mathematical necessity for high-dimensional FHE deployment.

Gradient inversion attacks present a significant privacy vulnerability in client-server machine learning. Because the HAO framework inherently requires a gradient clipping bound $G$ to stabilize the Bellman drift, it naturally provides the L2 sensitivity parameter required by the Gaussian mechanism of DP-SGD. 

\begin{figure}[h]
\centering
\includegraphics[width=1.0\columnwidth]{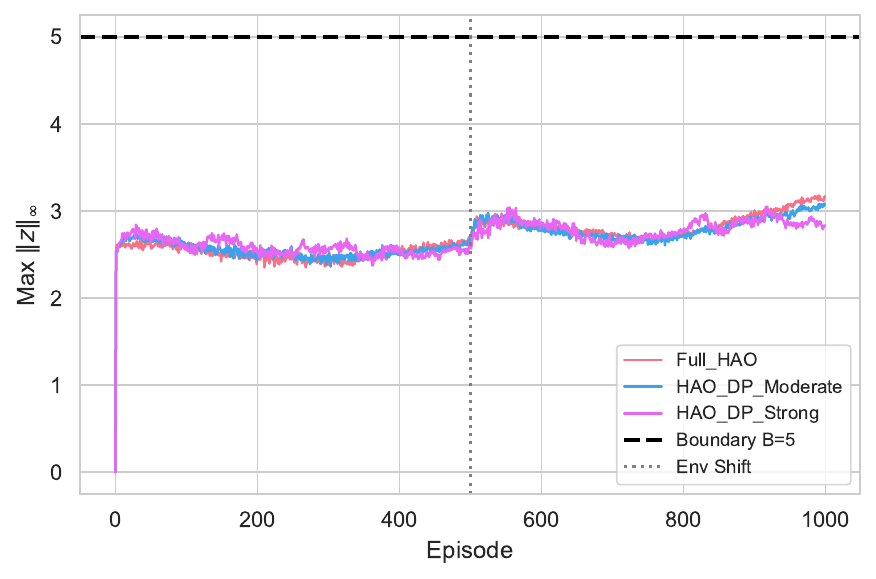}
\caption{Pre-activation robustness under DP noise. HAO successfully bounds the trajectory despite strong DP noise.}
\label{fig:exp3_dp_preact}
\end{figure}

\begin{figure}[h]
\centering
\includegraphics[width=1.0\columnwidth]{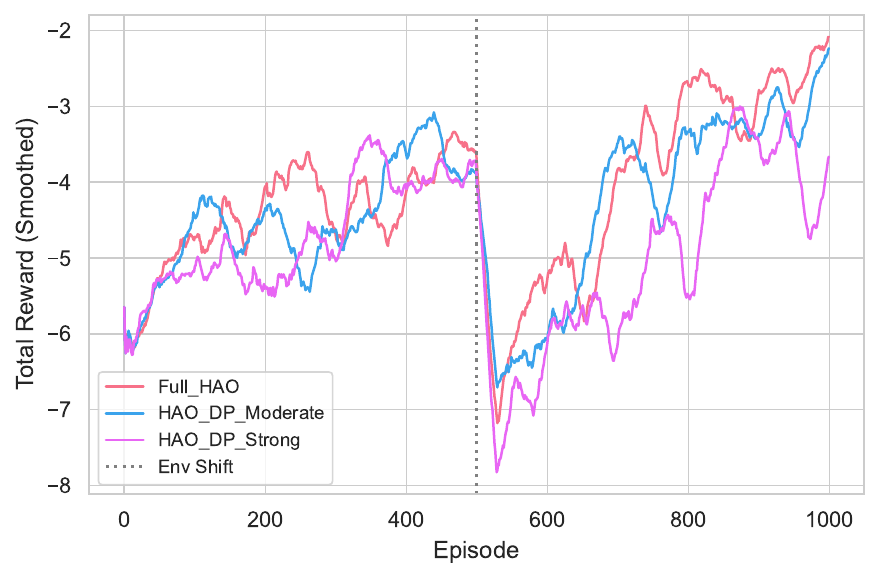}
\caption{Agent performance with DP noise injected. The privacy guarantee trades off only a marginal reduction in reward output.}
\label{fig:exp3_dp_reward}
\end{figure}

As shown in \Cref{tab:exp3} and \Cref{fig:exp3_dp_preact,fig:exp3_dp_reward}, integrating HAO with both moderate ($\sigma=0.3$) and strong ($\sigma=1.0$) DP-style noise maintains the perfect 0\% boundary breach rate. Because the injected DP noise is strictly zero-mean, it introduces no directional bias and does not disrupt the symmetric stabilization achieved by the centering operation. The marginally lower maximum pre-activation observed under strong noise ($4.01$ vs.\ $4.12$) reflects the regularizing perturbation of isotropic Gaussian noise, which disrupts systematic parameter drift along any single coordinate. The cost of this added noise is modest, incurring only a marginal reduction in raw reward output (from $-2.2$ down to $-4.1$), indicating that HAO tolerates DP-style gradient noise.

\section{Discussion}
\label{sec:discussion}
The empirical results across all three experimental domains confirm that standard deep RL stabilization techniques, namely, gradient clipping and L2 regularization, are fundamentally insufficient for FHE deployment. The HAO resolves this critical bottleneck not through heuristic hyperparameter tuning, but through a structural mathematical correction to the Bellman drift error.

Our empirical evaluations confirm that the unanchored drift of the state-value baseline $V(s)$ acts as a massive, asymmetric force that systematically overpowers symmetric stabilization techniques like L2 weight decay as observed in \Cref{fig:exp3_preact}. By deploying the $P$-matrix projection, the HAO framework structurally neutralizes this drift at its source. Isolating the advantage function $A(s,a)$ effectively collapses the magnitude of the signal flowing through the polynomial activation, guaranteeing that pre-activations remain bounded within $[-B, B]$ without requiring heuristic gradient penalties (\Cref{tab:exp3}).

Crucially, the Centering Only ablation still suffered boundary breaches in the high-dimensional logistics environment. While centering successfully stops the asymmetrical drift of $V(s)$, an unregularized neural network will still naturally and symmetrically expand its weights over time as it learns (\Cref{tab:exp3} and \Cref{fig:exp3_preact}). This confirms that the $P$-matrix is not a replacement for standard machine learning regularization. Rather, it is the mathematical prerequisite that finally allows standard regularization to function correctly in an FHE context.

Our findings also provide critical insight into the computational cost structure of recent encrypted RL architectures. Notably, Nguyen et al.~\cite{nguyen2025nmi} achieved encrypted deep RL by pairing an HE-compatible Adam optimizer with degree-15 and degree-31 minimax polynomial activations. However, supporting such high polynomial degrees necessitates full FHE bootstrapping, resulting in a computational latency of $\approx 505$\,s per single training step on CartPole~\cite{nguyen2025nmi}. We observe that these high-degree polynomials are fundamentally required to create an artificially wide convergence window that tolerates unconstrained pre-activation growth. By contrast, HAO directly attacks the root cause of this growth by annihilating the $V(s)$ drift at its mathematical source, and bounding pre-activations within $[-5, 5]$ using a lightweight degree-2 polynomial. Thus, it eliminates the need for bootstrapping, and its encrypted centering step takes $1.64$\,s per episode in our CartPole evaluations. This figure times the $[[Q]] \cdot P$ step alone, so it is not a like-for-like comparison with a full bootstrapped training step. The two approaches are complementary: HAO provides the activation-domain stabilization necessary for low-degree leveled HE, while optimizer reparameterization~\cite{nguyen2025nmi} resolves gradient-level bottlenecks.

In standard FHE-RL, calibrating the reward scale ($\alpha$) is traditionally the sole mechanism available to researchers to prevent polynomial explosion~\cite{sun2026privacy}. By testing HAO across a massive sweep of $\alpha \in [0.01, 0.25]$, we demonstrated that the framework maintains strict operational stability ($\|z\|_\infty < B$) across a $4\times$ wider operational window ($\alpha \in [0.01, 0.08]$) compared to regularization alone (\Cref{tab:reward_scale}). By drastically relaxing this rigid hyperparameter tuning requirement, HAO makes FHE-RL practical to deploy across diverse environments without exhaustive trial-and-error.

\begin{table}[h]
\centering
\caption{Maximum Pre-Activation ($\|z\|_\infty$) Across Reward Scale ($\alpha$) Sweeps. Values $>5.0$ trigger divergence.}
\label{tab:reward_scale}
\resizebox{\columnwidth}{!}{%
\begin{tabular}{lcccccc}
\toprule
$\alpha \rightarrow$ & $\mathbf{0.01}$ & $\mathbf{0.02}$ & $\mathbf{0.05}$ & $\mathbf{0.08}$ & $\mathbf{0.10}$ & $\mathbf{0.25}$ \\
\midrule
Reg. Only & 2.14 & 4.91 & Breach & Breach & Breach & Breach \\
Full HAO  & 1.25 & 1.87 & 2.95   & 4.88   & Breach & Breach \\
\bottomrule
\end{tabular}%
}
\end{table}

It is worth noting that adaptive value normalization techniques, such as PopArt~\cite{vanhasselt2016popart}, are widely used in plaintext deep RL to normalize targets across widely varying reward scales. However, PopArt relies on tracking online statistics and dividing targets by the running standard deviation, a non-polynomial operation that is incompatible with the arithmetic constraints of leveled homomorphic encryption. HAO achieves an analogous stabilizing effect through a purely linear, additive projection matrix $P$, serving as an HE-native counterpart to return normalization that incurs zero multiplicative depth.

The HAO framework is architecturally agnostic and inherently synergistic with DP. Because the centering projection is mathematically neutral and introduces no directional bias, the framework cleanly absorbs zero-mean Gaussian noise without destabilizing the polynomial bounds (\Cref{tab:exp3} and \Cref{fig:exp3_dp_preact}). In our implementation, this supports two distinct privacy tiers: (1)~CxP forward evaluation for maximum throughput, where the client keeps their state private; and (2)~CxP combined with Gaussian noise on the weight updates, which, with per-example clipping and privacy accounting, can provide formal $(\epsilon, \delta)$-privacy guarantees against gradient inversion attacks. Furthermore, the HAO centering matrix $P$ is a constant linear projection; therefore, it can be applied either on the client side or directly on the server without any additional ciphertext--ciphertext multiplication, meaning the mathematical framework theoretically extends to full Ciphertext-Ciphertext (CxC) evaluation for strict model confidentiality without increasing the polynomial degree overhead.

While HAO successfully anchors the Bellman drift, our current implementation is constrained to a single hidden-layer network and a degree-2 polynomial activation. These constraints are strictly dictated by the restrictive multiplicative depth limitations inherent to the TenSEAL CKKS encryption scheme. Scaling this framework to deeper architectures or higher-degree polynomials will require the integration of FHE bootstrapping, a computationally expensive operation that remains a central challenge for cryptographic machine learning. Furthermore, while HAO expands the viable operational window for reward scaling, $\alpha$ remains an environment-specific hyperparameter that requires baseline calibration.

\section{Conclusion and Future Work}
\label{sec:conclusion}

We identified the Bellman drift as a fundamental obstacle to deploying RL under FHE polynomial constraints and proposed the Homomorphic Advantage Operator (HAO) to solve it. Through centering TD targets with a zero-mean projection, HAO eliminates the persistent state-value bias that drives pre-activation growth. In empirical evaluations, HAO achieved 0\% boundary breaches while regularization alone breached the bound on 3 of 5 seeds, and simultaneously expanded the viable operational window for reward scaling by $4\times$. Because it requires no additional non-linear depth and is mathematically neutral, HAO is inherently HE-compatible and tolerates DP-style gradient noise. Thus, this makes it a highly practical and necessary component for privacy-preserving RL.

Future work must address the two primary frontiers of FHE-RL deployment: architecture depth and model confidentiality. Overcoming the strict multiplicative depth limits of shallow networks will require the integration of FHE bootstrapping, an open challenge that would enable scaling to deep, multi-layer topologies. For deployments requiring strict model confidentiality, extending the framework to Ciphertext-Ciphertext (CxC) evaluation is a natural next step. Because the HAO centering matrix $P$ is a constant linear projection that consumes no multiplicative depth, it can be seamlessly integrated directly onto the server's ciphertext pipeline, making it perfectly suited for future CxC paradigms.

\bibliographystyle{cas-model2-names}
\bibliography{bibliography}

\end{document}

%% file: architecture.tex
\begin{tikzpicture}[
    >=Stealth,
    font=\small,
    node distance=1.8cm,
    client_box/.style={rectangle, draw=clientblue, thick, fill=clientfill, 
                       minimum width=4.2cm, minimum height=1.1cm, align=center, rounded corners=3pt},
    server_box/.style={rectangle, draw=servergray, thick, fill=serverfill, 
                       minimum width=4.5cm, minimum height=1.1cm, align=center, rounded corners=3pt},
    hao_box/.style={rectangle, draw=haogold, very thick, fill=haofill, 
                    minimum width=4.2cm, minimum height=1.1cm, align=center, rounded corners=3pt},
    weight_box/.style={rectangle, draw=servergray, thick, fill=yellow!8, 
                       minimum width=4.5cm, minimum height=1.1cm, align=center, rounded corners=3pt},
    data_flow/.style={->, thick, draw=black!70},
    enc_flow/.style={->, very thick, draw=encgreen},
    dp_flow/.style={->, thick, draw=dpviolet, densely dashed},
    note/.style={font=\scriptsize\itshape, text=black!60, align=center},
]

\node[client_box, fill=clientblue!15] (env) at (0, 0) 
    {\textbf{Logistics Environment}\\{\scriptsize State $s \in \mathbb{R}^8$ (private)}};

\node[client_box, below=0.8cm of env] (scale) 
    {\textbf{Reward Scaling}\\{\scriptsize $r \leftarrow r \times \alpha$}};

\node[client_box, below=0.8cm of scale] (encrypt) 
    {\textbf{CKKS Encrypt}\\$s \rightarrow [\![s]\!]$};

\node[client_box, below=3.8cm of encrypt] (decrypt) 
    {\textbf{CKKS Decrypt}\\$[\![\tilde{Q}]\!] \rightarrow \tilde{Q}(s, \cdot)$};

\node[client_box, below=0.8cm of decrypt] (td) 
    {\textbf{Action + TD Target}\\$a = \arg\max \tilde{Q}$, \; $y_t = r + \gamma \max \tilde{Q}'$};

\node[client_box, fill=dpviolet!8, draw=dpviolet, below=0.8cm of td] (dpsgd)
    {\textbf{Backprop + DP-SGD + WD}\\{\scriptsize Clip($G$), $\mathcal{N}(0, \sigma^2 G^2)$, decoupled $\lambda$}};

\draw[data_flow] (env) -- (scale);
\draw[data_flow] (scale) -- (encrypt);
\draw[data_flow] (decrypt) -- (td);
\draw[data_flow] (td) -- (dpsgd);

\begin{scope}[on background layer]
    \node[draw=clientblue, thick, rounded corners=6pt, fill=clientblue!3, 
          inner sep=12pt, fit=(env)(scale)(encrypt)(decrypt)(td)(dpsgd)] (cframe) {};
    \node[above right, font=\bfseries\small, text=clientblue] at (cframe.north west) 
        {CLIENT (Private)};
\end{scope}

\node[weight_box] (weights) at (8.5, -0.0) 
    {\textbf{Stored Weights}\\$W_1, b_1, W_2, b_2$};

\node[server_box, below=0.8cm of weights] (fwd)
    {\textbf{Encrypted Forward Pass}\\$[\![z]\!] = [\![s]\!] \cdot W_1 + b_1$};

\node[server_box, below=0.8cm of fwd] (cheb)
    {\textbf{Polynomial Activation}\\$[\![a]\!] = P_2([\![z]\!])$\quad{\scriptsize$\|z\|_\infty \leq B$}};

\node[server_box, below=0.8cm of cheb] (qout)
    {\textbf{Output Layer}\\$[\![Q]\!] = [\![a]\!] \cdot W_2 + b_2$};

\node[hao_box, below=0.8cm of qout] (hao) 
    {\textbf{HAO Centering}\\$[\![\tilde{Q}]\!] = [\![Q]\!] \cdot P$\quad{\scriptsize(ciph $\times$ plain)}};

\node[weight_box, below=0.8cm of hao] (update)
    {\textbf{Weight Update (SGD)}\\$W \leftarrow W - \eta \cdot \Delta W - \lambda W$};

\draw[data_flow] (weights) -- (fwd) node[midway, right, note] {reads $W_1, b_1$};
\draw[data_flow] (fwd) -- (cheb);
\draw[data_flow] (cheb) -- (qout) node[midway, right, note] {reads $W_2, b_2$};
\draw[data_flow] (qout) -- (hao);

\draw[data_flow, thick, draw=orange!70!black] (update.east) -- ++(0.8, 0) |- (weights.east) 
    node[pos=0.5, right, note, text=orange!70!black] {updates\\$W_1, W_2$\\$b_1, b_2$};

\begin{scope}[on background layer]
    \node[draw=servergray, thick, rounded corners=6pt, fill=servergray!3, 
          inner sep=12pt, fit=(weights)(fwd)(cheb)(qout)(hao)(update)] (sframe) {};
    \node[above right, font=\bfseries\small, text=servergray] at (sframe.north west) 
        {SERVER (Untrusted)};
\end{scope}

\draw[enc_flow] (encrypt.east) -- ++(3.0, 0) |- (fwd.west) 
    node[pos=0.25, above left, font=\scriptsize\bfseries, text=encgreen] {$[\![s]\!]$}
    node[pos=0.25, below left, note, text=encgreen] {encrypted state};

\draw[enc_flow] (hao.west) -- ++(-1.2, 0) |- (decrypt.east)
    node[pos=-0.1, above left, font=\scriptsize\bfseries, text=encgreen] {$[\![\tilde{Q}]\!]$}
    node[pos=-0.2, below left, note, text=encgreen] {centered Q-values};

\draw[dp_flow] (dpsgd.east) -- ++(1.5, 0) |- (update.west)
    node[pos=0.0, above right, font=\scriptsize\bfseries, text=dpviolet] {$\Delta W$}
    node[pos=0.0, below right, note, text=dpviolet] {DP-protected};

\end{tikzpicture}